\documentclass[lettersize,journal]{IEEEtran}
\usepackage{amsmath,amsfonts}
\usepackage{algorithmic}
\usepackage{algorithm}
\usepackage{array}
\usepackage[caption=false,font=normalsize,labelfont=sf,textfont=sf]{subfig}
\usepackage{textcomp}
\usepackage{stfloats}
\usepackage{url}
\usepackage{verbatim}
\usepackage{graphicx}
\usepackage{cite}
\usepackage{pifont}
\usepackage{bm}
\usepackage{amsthm}
\usepackage{multirow}
\usepackage{booktabs}
\begin{document}
	
	\title{Distribution Fitting for Combating Mode Collapse \\in Generative Adversarial Networks}
	
	\author{Yanxiang Gong,~\IEEEmembership{Student Member,~IEEE,}
		Zhiwei Xie,
		Guozhen Duan,
		Zheng Ma,
		Mei Xie
		\thanks{Corresponding author: Mei Xie.}
		\thanks{The authors are with the School of Information and Communication Engineering, University of Electronic Science and Technology of China, Chengdu 611731, China (e-mail: yxgong@std.uestc.edu.cn; xiezhiwei@std.uestc.edu.cn; 202022011405@std.uestc.edu.cn; zma@uestc.edu.cn; mxie@uestc.edu.cn).}
	}

\markboth{Journal of \LaTeX\ Class Files,~Vol.~14, No.~8, August~2021}%
{Shell \MakeLowercase{\textit{et al.}}: A Sample Article Using IEEEtran.cls for IEEE Journals}

\IEEEpubid{0000--0000/00\$00.00~\copyright~2021 IEEE}
\maketitle

\begin{abstract}
	Mode collapse is a significant unsolved issue of generative adversarial networks. In this work, we examine the causes of mode collapse from a novel perspective. Due to the nonuniform sampling in the training process, some sub-distributions may be missed when sampling data. As a result, even when the generated distribution differs from the real one, the GAN objective can still achieve the minimum. To address the issue, we propose a global distribution fitting (GDF) method with a penalty term to confine the generated data distribution. When the generated distribution differs from the real one, GDF will make the objective harder to reach the minimal value, while the original global minimum is not changed. To deal with the circumstance when the overall real data is unreachable, we also propose a local distribution fitting (LDF) method. Experiments on several benchmarks demonstrate the effectiveness and competitive performance of GDF and LDF.
	
\end{abstract}



\section{Introduction}
\label{sec:intro}
\IEEEPARstart{G}{enerative} adversarial networks (GANs)~\cite{gan} have attracted much attention since the first introduction. Owing to its excellent performance, this adversarial training scheme became one of the most significant achievements in the machine learning field of the past decades. The framework simultaneously trains two sub-networks: a generator that generates samples from a noise vector and a discriminator designed to distinguish real and fake samples. In the optimal case, the generated samples cannot be discriminated from the real ones. However, GAN training is dynamic and sensitive to nearly every aspect of its setup, from optimization parameters to model architecture~\cite{biggan}. Thus, GANs encounter challenges in the training process. Mode collapse is one of the main obstacles, which represents that GANs may fail to model the full distribution and some minor modes might not be captured. Even a simple Gaussian mixture model can work better than GANs on capturing the underlying distribution and provides an explicit representation of the statistical structure~\cite{gmms}. To address the issue, researchers presented various methods, including applying multi-generators~\cite{madgan,mgogan}, improving the training strategy~\cite{unrolledgan,adagan,lsvgd,bayesiancyclegan}, using penalty modules~\cite{pacgan,svdd, veegan,dpm,mggan} and adding constraints~\cite{wgan,wgangp,dragan,distgan,sngan,srgan,mescheder2018training,dualaae,msgan,mdgan}. There are also methods for unbalanced or very small samples under different categories~\cite{salazar2021generative,grover2019bias,yu2020inclusive}. However, mode collapse remains a troublesome issue in complicated applications of GANs. Researchers are still working hard on it~\cite{divco,uctgan,maskgan,stargan,divgan}. Therefore, the study of methods for combating mode collapse remains of great importance.

In GANs~\cite{gan}, the output of the discriminator represents the probability that the input sample came from the real data. The authors demonstrated that the generator's global minimum under the optimal discriminator is reached if and only if the generated distribution is the same as the real one. However, we notice that the minimum value is also reachable while the generated and real distributions are not the same due to the nonuniform sampling in the training process. In this case, the sampled distribution cannot adequately denote the complete real distribution, and the GAN objective can occasionally reach the minimum in practical training. As a result, the risks of mode collapse will increase and GAN training will be sensitive, requiring us to adjust each parameter carefully. Therefore, we consider that the nonuniform sampling issue is one of the reasons which leads to mode collapse. However, the existing techniques mostly aim at other aspects of the mode collapse issues, which cannot handle the mentioned problem well. Meanwhile, some disadvantages are also obstacles to applications of the existing methods, such as heavy computation, requirements of manual design, and lack of interpretability. Thus, related research on this problem is still valuable and indispensable.

In this work, we propose a global distribution fitting method to combat the mode collapse problem. We design a penalty to constrain the generated data's distribution and incorporate it into the training objective. We prove that on the basis of not changing the global minimum of the GAN objective, the method will make it harder to reach the minimum value while the generated and real distributions are not the same. Furthermore, to cope with situations where the overall real data is unreachable, we propose a local distribution fitting method, utilizing sampled distribution to match the real one. The proposed methods will not bring heavy computation or require complicated manual design, making it easy to implement them in other frameworks. The experiments on several benchmarks demonstrate the effectiveness and competitive performance. To summarize, there are three main contributions in this work:

\IEEEpubidadjcol

\begin{itemize}
	\item We analyze the causes of mode collapse from a new perspective. Because of the nonuniform sampling in practical training, the GAN objective could reach the minimum value when the generated and real distributions are not the same. Therefore, the training will be sensitive and risks of mode collapse will increase.
	\item We propose global and local distribution fitting methods to suppress mode collapse. On the basis of not changing the global minimum of the GAN objective, the method will make the objective harder to reach the minimum when the generated distribution is not the same as the real one.
	\item We demonstrate the effectiveness of the proposed method with both theoretical derivation and experiments. The theorems and proofs are also of great reference value for other methods based on distribution or feature matching.
\end{itemize}

\section{Related Work}
\label{sec:related}
Research on mode collapse problems has been an attractive field in recent years. Existing methods can be roughly classified into four categories: applying multi-generators, improving training, using penalty modules, and adding constraints. Each kind of method is effective and aims at different aspects of the issue.

\subsection{Applying Multi-Generators}
A simple but effective method is to apply multi-generators in GANs. For instance, Ghosh~\textsl{et al.}~\cite{madgan} proposed to utilize multi-generators to generate different data. Li~\textsl{et al.}~\cite{mgogan} also applied multiple generators to learn the mapping from the randomized space to the original data space. The performance is excellent, but a manual design is required for the number of generators. Therefore, these methods can be used to maximum advantage in tasks where the mode number is known.

\subsection{Improving Training}
Some researchers tried to improve the training strategies to suppress mode collapse. Metz~\textsl{et al.}~\cite{unrolledgan} proposed to "unroll" the optimizer to suppress mode collapse. Tolstikhin~\textsl{et al.}~\cite{adagan} borrowed some ideas from AdaBoost to address the problem of missing modes. Wang~\textsl{et al.}~\cite{lsvgd} proposed an optimization method named Langevin Stein Variational Gradient Descent to stabilize the training. You~\textsl{et al.}~\cite{bayesiancyclegan} proposed to optimize the model with maximum a posteriori estimation. The advantage of these methods is that they only change the training strategy, so they can be easily applied in any framework. However, additional computation will be required, and thereby they might not be efficient under restricted hardware conditions.

\subsection{Using Penalty Modules}
Using penalty modules is also an idea for suppressing mode collapse. The central idea of these methods is to penalize the generator when the generated data diversity is insufficient. Srivastava~\textsl{et al.}~\cite{veegan} featured a reconstructor network to reverse the action of the generator by mapping from data to noise. Lin~\textsl{et al.}~\cite{pacgan} modified the discriminator to make decisions based on multiple samples from the same class, which penalizes generators with mode collapse. Chong~\textsl{et al.}~\cite{svdd} utilized a regularizer to penalize the mini-batch variance. Pei~\textsl{et al.}~\cite{dpm} introduced a pluggable diversity penalty module to enforce the generator to generate images with distinct features. Bang~\textsl{et al.}~\cite{mggan} leveraged a guidance network to induce the generator to learn the overall modes. The methods directly require the generated data to be diverse, which achieve competitive performance.

\subsection{Adding Constraints}
Adding constraints is another simple and effective way to suppress mode collapse. Arjovsky~\textsl{et al.}~\cite{wgan} first proposed to clip the gradient to satisfy Lipschitz continuity. This idea brought lots of related research. Then Gulrajani~\textsl{et al.}~\cite{wgangp} improved the method by introducing gradient penalty, which is more stable. Kodali~\textsl{et al.}~\cite{dragan} also proposed to introduce the penalty to limit the gradient around some real data points. Miyato~\textsl{et al.}~\cite{sngan} proposed to constrain the spectral norm of each network layer to control the Lipschitz constant. Liu~\textsl{et al.}~\cite{srgan} found an issue of spectral collapse and proposed a method to alleviate the issue. Mescheder~\textsl{et al.}~\cite{mescheder2018training} analyzed the regularization strategies and proved local convergence for simplified gradient penalties. Ge~\textsl{et al.}~\cite{dualaae} introduced the clustering regularization term in the network to avoid mode collapse. The aim of these methods is to limit the gradient because the gradient will be sharp when mode collapse occurs~\cite{dragan}.

In addition to limiting gradients, other constraints were also applied. Tran~\textsl{et al.}~\cite{distgan} proposed a distance constraint to enforce compatibility between the latent sample distances and the corresponding data sample distances. Mao~\textsl{et al.}~\cite{msgan} proposed maximizing the distance ratio between generated images. Eghbal-zadeh~\textsl{et al.}~\cite{mdgan} encouraged the discriminator to form clusters, which will lead the generator to discover different modes. Compared to limiting gradients, these methods explicitly specify the generated diversity. In this work, we propose a new constraint from a novel perspective. The parameters of the real and generated distributions are constrained to be almost the same. Thus, to maintain the consistency of parameters, the data modes cannot be missed.

\section{Analysis of Mode Collapse}
The training objective of the original GANs~\cite{gan} is
\begin{equation}\label{eq1}
	\begin{aligned}
		\min_G\max_DV(D,G)=&\mathbb{E}_{\bm{x}\sim p_{data}(\bm{x})}[log(D(\bm{x}))]+\\
		&\mathbb{E}_{\bm{z}\sim p_{z}(\bm{z})}[log(1-D(G(\bm{z})))]
	\end{aligned}
\end{equation}
where $\bm{x}$ is a real sample and $\bm{z}$ is a random noise vector. The authors have demonstrated that the optimal $D$ is
\begin{equation}
	D_G^*(\bm{x})=\frac{p_{data}(\bm{x})}{p_{data}(\bm{x}) + p_{g}(\bm{x})}
\end{equation}
where $p_{data}$ denotes the real data distribution and $p_g$ denotes the generated data distribution. The virtual training criterion 
\begin{equation}
	C(G)=\max_DV(D,G)
\end{equation}
which is a Jensen–Shannon divergence, will reach the global minimum $-log(4)$ when $p_g=p_{data}$. At that point, $D_G^*(\bm{x})=\frac{1}{2}$. In later studies, there are also improved frameworks that utilized Wasserstein distance~\cite{wgan,wgangp} or Hinge loss~\cite{lim2017geometric}, which are also proven to reach the minimum if and only if $p_g=p_{data}$ and effectively learn the distribution.

However, there may be some issues. Because we use a sampled batch of data whose distribution is $p_s$ to replace $p_{data}$ in practical training, nonuniform sampling may occur in some cases. Thus, we conjecture that the minimum value can be reached when $p_g\neq p_{data}$ in some cases.

First, we will give the mathematical definition of nonuniform sampling. The prerequisite of nonuniform sampling is that $p_{data}$ is a mixed distribution that consists of $\bm{\phi}(\bm{x})=\{\phi_1(\bm{x}),\phi_2(\bm{x}),...,\phi_K(\bm{x})\},~K>1$, which can be linearly expressed as
\begin{equation}\label{eq4}
	p_{data}(\bm{x})=\sum_{k=1}^{K}\alpha_k\phi_k(\bm{x})
\end{equation} 
where $\alpha_k\in(0,1)$ is the weight of the $k$-th sub-distribution.

\newtheorem{definition}{Definition}
\begin{definition}
	In one training iteration of GANs, if each real data in a batch is sampled from one of, or the combination of several sub-distributions $\{\phi_m(\bm{x}),\phi_{m+1}(\bm{x}),...,\phi_{n}(\bm{x})\}$ of $p_{data}$, where $1\le m\le n\le K$ and $n-m<K-1$, nonuniform sampling occurs.
\end{definition}

According to the definition, it can be observed that some sub-distributions are missed when nonuniform sampling occurs because no data is sampled from them. Thus, the training objective can reach the minimum in some iterations when $p_g\neq p_{data}$.

\newtheorem{lemma}{Lemma}[]
\begin{lemma} \label{lemma1}
	In the training process of GANs, $C(G)$ can sometimes reach the minimum even if $p_g\neq p_{data}$.
\end{lemma}
\begin{proof}
	In one training iteration of GANs, we will sample a batch of data to represent $p_{data}$, whose distribution is $p_s$. When nonuniform sampling occurs, the sampled data distribution can be expressed as
	\begin{equation}\label{eq5}
		p_{s}(\bm{x})=\sum_{k=m}^{n}\beta_k\phi_k(\bm{x})
	\end{equation}
	where $\beta_k\in (0,1)$ is the weight of the $k$-th sub-distribution. Obviously, it is possible that $p_s\neq p_{data}$ because $n-m < K-1$. Then if $p_{g}$ only learns the sampled distribution which achieves $p_g=p_{s}$, $C(G)$ will reach the minimum when $p_g\neq p_{data}$.
\end{proof}

Lemma~\ref{lemma1} indicates that GAN loss can reach the minimum when the generator is not well trained. In this case, the generator only learns a part of the real distribution, which brings risks of mode collapse. We can identify several circumstances where nonuniform sampling is more likely to occur based on the lemma.

\newtheorem{thm}{Theorem}[]
\begin{thm}\label{thm1}
	In the training process, nonuniform sampling is more likely to occur when the overlaps between sub-distributions are smaller.
\end{thm}
\begin{proof}15
	In $p_{data}$, when there is an overlap between several sub-distributions, a data sample can be sampled from their combination. Without loss of generality, assume that $p_{data}$ has only two sub-distributions that contain an overlap. Use $\bm{x}\in overlap$ to represent $\bm{x}$ is sampled from the combination of the two sub-distributions, and then the probability can be marked as $P(\bm{x}\in overlap|\bm{x}\in p_{data})$.
	
	If nonuniform sampling occurs, each data in a real data batch should be sampled from one sub-distribution according to the definition. The probability that all samples are from one sub-distribution is
	\begin{equation}
		\begin{aligned}
			&\prod_{i=0}^{b} P(\bm{x}_i\notin overlap|\bm{x}_i\in p_{data})\\
			= &\prod_{i=0}^{b}[1 - P(\bm{x}_i\in overlap|\bm{x}_i\in p_{data})]
		\end{aligned}
	\end{equation}
	where $\bm{x}_i$ is the $i$-th sample in a batch and $b$ is the batch size. If the overlaps between sub-distributions are smaller, we will have $P(\bm{x}_i\in overlap|\bm{x}_i\in p_{data})\to 0$, which indicates $P(\bm{x}_i\notin overlap|\bm{x}_i\in p_{data})\to 1$. Therefore, the probability that all samples are from one sub-distribution is larger, and nonuniform sampling is more likely to occur. It is not difficult to generalize to the scenes with more sub-distributions.
\end{proof}

\begin{thm}\label{thm2}
	In the training process, nonuniform sampling is more likely to occur when one or several sub-distributions are dominating in the sampling of $p_{data}$ in the training process.
\end{thm}
\begin{proof}
	
	Use $\bm{x}\in p$ to represent $\bm{x}$ is sampled from distribution $p$. Let the sub-distributions $\phi_{s}=\{\phi_m,\phi_{m+1},...,\phi_{n}\}$ to be dominating in $p_{data}$ in the sampling of training, where $1\le m\le n\le K$ and $n-m<K-1$. Then it is equivalent to the situation that $P(\bm{x}_i\in \phi_{s}|\bm{x}_i\in p_{data})\to 1$. Then in a data batch, we will have
	\begin{equation}
		\sum_{i=1}^{b}P(\bm{x}_i\in \phi_{s}|\bm{x}_i\in p_{data})\to 1.
	\end{equation}
	According to the definition, nonuniform sampling is more likely to occur. 
\end{proof}

\newtheorem{corollary}{Corollary}[]
\begin{corollary} \label{cor1}
	In the training process, a large batch size can alleviate the nonuniform sampling problem.
\end{corollary}
\begin{proof}  
	According to the proof of Theorem~\ref{thm2}, the probability 15that all samples are from $\phi_s$ is
	\begin{equation}
		\prod_{i=1}^{b}P(\bm{x}_i\in\phi_s|\bm{x}_i\in p_{data})
	\end{equation}
	Because the range of $P(\bm{x}_i\in\phi_s|\bm{x}_i\in p_{data})$ is $(0,1)$, the probability of nonuniform sampling reduces with a larger $b$.
\end{proof}

At the points where nonuniform sampling occurs, $C(G)$ will be more likely to reach the minimum when $p_g\neq p_{data}$. Therefore, there will be more risks of mode collapse. We consider that most distributions contain linearly combined parts, so nonuniform sampling is one reason for mode collapse, and we will try to alleviate this problem.


\section{Method}
\label{sec:method}
\subsection{Motivation}
As mentioned above, some samples can make $C(G)$ reach or be close to the minimum value when $p_g\neq p_{data}$. They are like pits on a flat surface, and we will try to fix them to combat mode collapse.

In the theoretical design of GAN frameworks, the virtual training criterion for the generator is a distance that reaches the minimum when $p_{g}=p_{data}$. However, when nonuniform sampling occurs, the network's vision is limited to $p_s$. No matter how accurate the method we use, only $p_{g}=p_{s}$ can be achieved because $p_{data}$ is unknown in this iteration. To alleviate this problem, trying to incorporate information of $p_{data}$ into each training iteration should be effective. Based on this idea, we propose two distribution fitting methods aiming at different conditions.

\subsection{Global Distribution Fitting}
A direct method is to estimate the overall real distribution, and require the generated distribution to match it. However, the distributions are usually complicated, whose probability density functions are difficult to represent. Meanwhile, estimating the overall generated distribution in training will also be time-consuming. Thus, we consider only involving some critical information instead of the overall distribution. For a distribution, the mean and standard deviation are critical parameters. Global distribution fitting (GDF) is a method of adding a penalty term to constrain the means and standard deviations. First, we calculate the mean and standard deviation of each point of the overall real data before training, which are represented as $\bm{\mu}_r=\{\mu_{r1},\mu_{r2},...,\mu_{rn}\}$ and $\bm{\sigma}_r=\{\sigma_{r1},\sigma_{r2},...,\sigma_{rn}\}$ for $n$-tuple data. Second, in the training process, the generator will generate a mini-batch data whose mean and standard deviation are $\bm{\mu}_g=\{\mu_{g1},\mu_{g2},...,\mu_{gn}\}$ and $\bm{\sigma}_g=\{\sigma_{g1},\sigma_{g2},...,\sigma_{gn}\}$. We use $\bm{\mu}_g$ and $\bm{\sigma}_g$ to match the mean and standard deviation of $p_g$. Then the penalty term is defined as
\begin{equation}\label{eq12}
	\mathbb{L}(G)=\frac{1}{n}[||\bm{\mu}_g-\bm{\mu}_r||_1+||\bm{\sigma}_g-\bm{\sigma}_r||_1]
\end{equation}
It is expected to be minimized by the generator, which is independent of the adversarial training. The total training objective is
\begin{equation}\label{eq10}
	\min_G[\max_DV(D,G)+\lambda\mathbb{L}(G)]
\end{equation}
$\lambda$ is the weight of the penalty term. We set it to $1$ in experiments because its value will not have significant impacts on the properties of GDF. Then we will prove that GDF will not change the global minimum of the GAN objective and that it can suppress the nonuniform sampling issue. Two lemmas are introduced at first.

\begin{lemma}\label{lemma41}
	For a function $f(x)=f_1(x)+f_2(x)$, $f(x)$ must be bounded below if $f_1(x)$ and $f_2(x)$ are both bounded below.
\end{lemma}
\begin{proof}
	Assume $f_1(x)\geq M$ and $f_2(x)\geq N$, then there must be $f(x)\geq M+N$. Therefore, $f(x)$ must be bounded below.
\end{proof}
\begin{lemma}\label{lemma42}
	For a function $f(x)=f_1(x)+f_2(x)$ where $f_1(x)$ and $f_2(x)$ are both bounded below, if $f_1(x)$ and $f_2(x)$ can reach the minimum simultaneously, $f(x)$ will reach the minimum if and only if $f_1(x)$ and $f_2(x)$ reach the minimum simultaneously.
\end{lemma}
\begin{proof}
	Assume $f_1(x)\geq M$ and $f_2(x)\geq N$, then there must be $f(x)\ge M+N$. If there exists $x$ which make $f_1(x)=M$ and $f_2(x)=N$, we will have $f(x)=M+N$, which is the minimum. Therefore, $f(x)$ will reach the minimum if $f_1(x)$ and $f_2(x)$ reach the minimum simultaneously.
	
	Assume there exists $x$ which makes $f(x)=M+N$ when $f_1(x)$ and $f_2(x)$ do not reach the minimum simultaneously. Then there must be $f_1(x) > M$ or $f_2(x) > N$. If $f_1(x) > M$, we have 
	\begin{equation}
		\begin{aligned}
			f_2(x) &= M + N - f_1(x)\\
			&=N+[M-f_1(x)] \\
			&<N
		\end{aligned}
	\end{equation}
	which is contrary with $f_2(x)\geq N$. If $f_2(x) > N$, similarly we have $f_1(x)< M$ which is contrary with $f_1(x)\geq M$. Therefore, $f(x)$ will reach the minimum only if $f_1(x)$ and $f_2(x)$ reach the minimum simultaneously.
\end{proof}

Then we will prove the effectiveness of GDF. Because $C(G)=max_DV(D,G)$, for simplicity, let $C^*(G)=C(G)+\lambda\mathbb{L}(G)$.
\begin{thm}\label{thm41}
	If and only if $p_g=p_{data}$, $C^*(G)$ can reach the global minimum.
\end{thm}
\begin{proof}
	We know that $C(G)$ is bounded below according to~\cite{gan}. In the calculation of $\mathbb{L}(G)$, the 1-norm is non-negative. According to Lemma~\ref{lemma41}, $\mathbb{L}(G)$ is bounded below. Then for $C^*(G)$, $C(G)$ and $\mathbb{L}(G)$ are both bounded below. According to Lemma~\ref{lemma41}, $C^*(G)$ is also bounded below.
	
	The minimum of $C(G)$ is reached when $p_g=p_{data}$ according to~\cite{gan}. For $\mathbb{L}(G)$, 1-norm will reach the minimum $0$ when all the elements in the vector is zero. And because it is possible that $\bm{\mu}_g=\bm{\mu}_r$ and $\bm{\sigma}_g=\bm{\sigma}_r$ simultaneously, the minimum of $\mathbb{L}(G)$ is $0$ according to Lemma~\ref{lemma42}. Then because it is possible that $p_g=p_{data}$, $\bm{\mu}_g=\bm{\mu}_r$ and $\bm{\sigma}_g=\bm{\sigma}_r$ exist simultaneously, $C^*(G)$ can reach the minimum according to Lemma~\ref{lemma42}.
	
	According to Lemma~\ref{lemma42}, $C^*(G)$ will reach the minimum only if $C(G)$ and $\mathbb{L}(G)$ reach the minimum simultaneously. According to~\cite{gan}, $C(G)$ will reach the minimum only if $p_g=p_{data}$. Therefore, $C^*(G)$ will reach the minimum only if $p_g=p_{data}$.
\end{proof}

According to Theorem~\ref{thm41}, the global minimum of the GAN training objective is not changed with GDF. For other GAN frameworks with different objectives, the theorem is also effective because they still reach the minimum if and only if $p_g=p_{data}$. Thus, there are no negative effects for GANs with our method. 

\begin{thm}\label{thm42}
	If nonuniform sampling occurs, $C^*(G)$ can reach the minimum only if the means of all sub-distributions of $p_{data}$ can be linearly expressed by basic solutions
	\begin{equation}
		(\xi_1,\xi_2,\xi_3,...,\xi_n)=
		\begin{pmatrix}
			A \\
			E
		\end{pmatrix}
	\end{equation}
	where $E$ is an $n-1$ dimension identity matrix. $A$ is a row vector
	
	\begin{equation}
		\left(
		-\frac{\alpha_2}{\alpha_1},-\frac{\alpha_3}{\alpha_1},...,-\frac{\alpha_m-\beta_m}{\alpha_1},...,-\frac{\alpha_n-\beta_n}{\alpha_1},...,-\frac{\alpha_K}{\alpha_1}
		\right)
	\end{equation}
	where $\alpha_k$ is the weight of the $k$-th sub-distribution of $p_{data}$, and $\beta_k$ is the weight of the $k$-th sub-distribution of $p_s$.
	
\end{thm} 
\begin{proof}
	Assuming that $C^*(G)$ reaches the minimum. First, $C(G)$ must reach the minimum according to Lemma~\ref{lemma42}. Then there must be $p_g=p_s$ because we use $p_s$ to represent $p_{data}$ in training. Therefore, we have $\bm{\mu}_g=\bm{\mu}_s$. Second, $||\bm{\mu}_g-\bm{\mu}_r||_1$ must be zero according to Lemma~\ref{lemma42}, which indicates $\bm{\mu}_g=\bm{\mu}_r$. As a result, we have $\bm{\mu}_s=\bm{\mu}_r$.
	
	From Equations~\ref{eq4} and~\ref{eq5}, we can calculate the means of $p_{data}$ and $p_{s}$
	\begin{equation}\label{eq14}
		\bm{\mu}_r = \sum_{k=1}^{K}\alpha_k\mu_{k}
	\end{equation}
	\begin{equation}\label{eq15}
		\bm{\mu}_s = \sum_{k=m}^{n}\beta_k\mu_{k}
	\end{equation}
	where $\mu_{k}$ is the mean of the $k$-th sub-distribution. With $\bm{\mu}_s=\bm{\mu}_r$, we have 
	\begin{equation}
		\begin{aligned}
			&\sum_{k=1}^{K}\alpha_k\mu_{k} = \sum_{k=m}^{n}\beta_k\mu_{k}\\
			\iff&\alpha_1\mu_1+\alpha_2\mu_2+...+(\alpha_m-\beta_m)\mu_m+...\\
			&+(\alpha_n-\beta_n)\mu_n+...+\alpha_K\mu_K = 0
		\end{aligned}
	\end{equation}
	whose basic solution set is $(A^T,E)^T$.
\end{proof}

It is not common that the means of the sub-distributions are in a certain proportional relationship, especially when $\beta_k$ is not static and influenced by the sampling process. Therefore, according to Theorem~\ref{thm42}, the minimum value is harder to be reached when $p_g\neq p_{data}$. Even if the sub-distributions are not linearly combined, the means should still be in a certain proportional relationship, and only the calculations in Eqs. \ref{eq14} and \ref{eq15} need to be changed. As a result, $p_g$ will be more likely to converge to $p_{data}$, and mode collapse will be suppressed.

\subsection{Local Distribution Fitting}
In some circumstances, the real distribution is unknown or hard to calculate. For instance, if the real data is a black box and we can only sample from it, GDF will not work. It is because calculating the distribution parameters before training will be difficult, as it is almost impossible to define a general criterion of minimum training sample size~\cite{salazar2023proxy} and we also cannot define the number of samples required for estimation. To cope with this situation, we propose a local distribution fitting (LDF) method. It utilizes the mean $\bm{\mu}_s$ and standard deviation $\bm{\sigma}_s$ of all sampled data in the previous training iterations to replace the real ones. In the training process, the penalty term is
\begin{equation}
	\mathbb{L}_s(G)=\frac{1}{n}[||\bm{\mu}_g-\bm{\mu}_s||_1+||\bm{\sigma}_g-\bm{\sigma}_s||_1]
\end{equation}
and the training objective is
\begin{equation}\label{eq18}
	\min_G[\max_DV(D,G)+\lambda_s\mathbb{L}_s(G)]
\end{equation}
$\lambda_s$ is the weight which is also set to $1$. When more and more data is sampled, $\bm{\mu}_s$ and $\bm{\sigma}_s$ will be closer to the real ones. Then according to Theorems~\ref{thm41} and~\ref{thm42}, $p_g$ will be more likely to converge to $p_{data}$, and the mode collapse problem will also be suppressed.

\section{Experiments}
\label{sec:experiments}
In this section, we conduct experiments on the proposed methods. Several datasets will be involved, and the evaluation metrics will be introduced in each section.

\subsection{Mixture of Gaussian Dataset}
To illustrate the impact of distribution fitting, we train a simple network on a 2D mixture of 8 Gaussians arranged in a circle. Both the training setup and the network architecture are the same as those proposed by unrolledGAN~\cite{unrolledgan}, and we utilize an implementation of it\footnote{https://github.com/andrewliao11/unrolled-gans} to conduct the experiment. The detailed network architecture is also provided in the Appendix. This experiment gives an intuitive result, which proves the effectiveness of GDF and LDF in suppressing mode collapse.

\begin{figure}[h]
	\centering
	\includegraphics[width=3.3in]{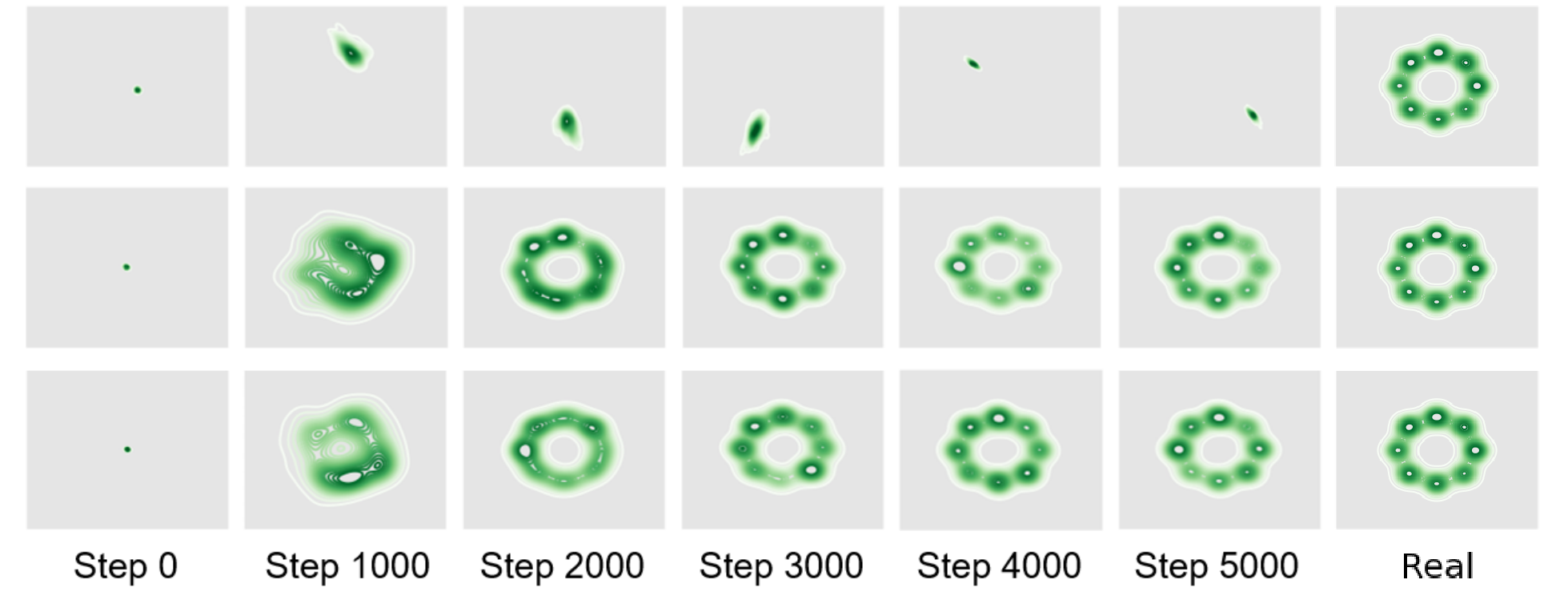}
	\caption{The dynamics of the model through time. The image rows from top to bottom represent the model without distribution fitting, the model with LDF, and the model with GDF, respectively.}
	\label{fig1}
\end{figure}

In Fig.~\ref{fig1}, we observe that the original model suffers from severe mode collapse. Using GDF or LDF will inhibit mode collapse, and the network can learn a complete distribution. However, the difference between GDF and LDF is marginal. It is because the batch size is 512 in the training process. The real data only consists of eight sub-distributions, so the distribution of a generated batch is already very close to the global one.

\begin{figure}[h]
	\centering
	\includegraphics[width=3.3in]{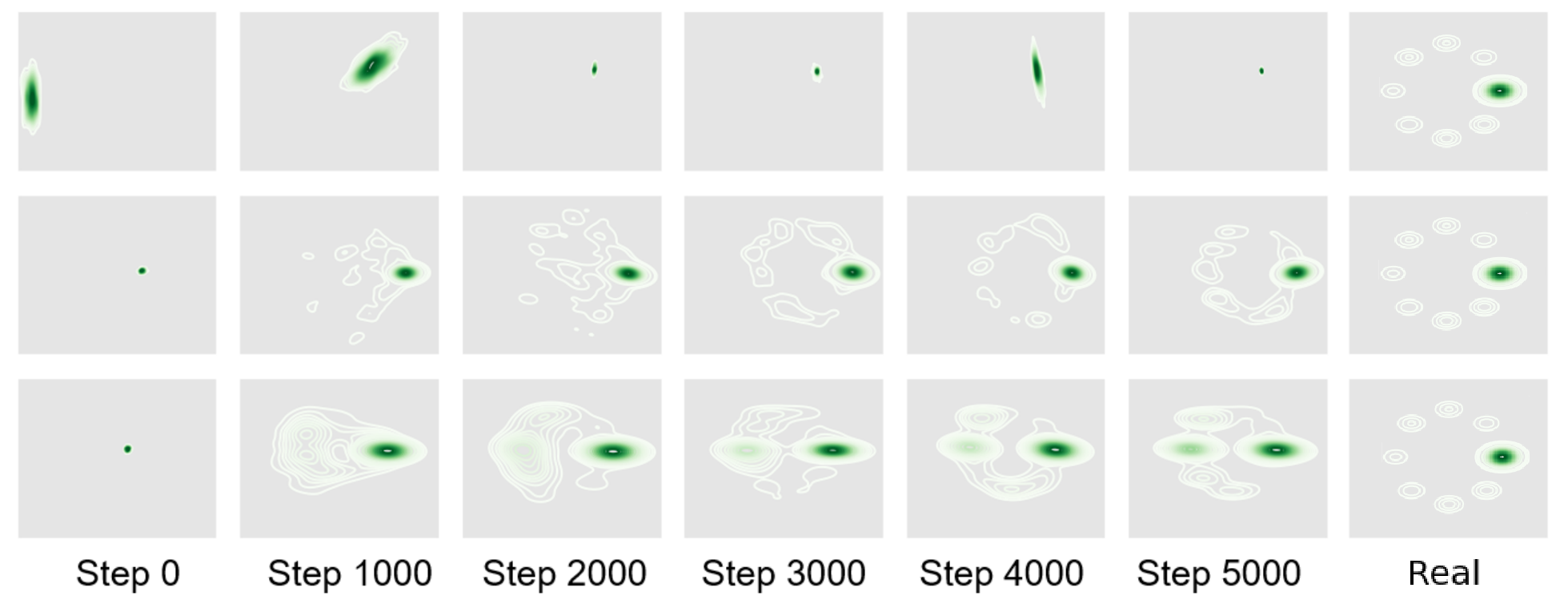}
	\caption{The dynamics of the model through time in an extreme case. The image rows from top to bottom represent the model without distribution fitting, the model with LDF, and the model with GDF, respectively.}
	\label{fig2}
\end{figure}

There is also an extreme case where one Gaussian is dominating with a large weight while the others account for only a small part. Although the distribution of the 8 Gaussians is nonuniform with the training setup, this case is still hard to reach. To demonstrate the effectiveness of GDF and LDF in this case, we manually set the weight of one Gaussian to $0.86$, while the other seven weights are all $0.02$. The training dynamics are illustrated in Fig.~\ref{fig2}. With GDF or LDF, the network can also generate samples from the other sub-distributions, indicating that mode collapse is also suppressed.

\subsection{Stacked MNIST Dataset}
The effectiveness of suppressing mode collapse is frequently evaluated through experiments conducted on the stacked MNIST dataset. According to~\cite{che2016mode}, the images in the stacked MNIST dataset are obtained by stacking three random single-channel images from MNIST~\cite{mnist}. Therefore, there are 1,000 modes that need to be generated. In practice, the stacked images are produced dynamically in the training process. As there are too many possible stacked images, the real mean and standard deviation cannot be calculated before training. Therefore, GDF is not available.

In order to demonstrate the effectiveness of both GDF and LDF, we first make a static stacked MNIST dataset containing 60,000 images for training, so the mean and standard deviation of the real distribution can be calculated. Considering the practice is different from previous methods, we do not involve other frameworks for comparisons in this experiment. The baseline has a DCGAN-like architecture which is the same as the one proposed by~\cite{pacgan}, which has been provided in the Appendix. Each network will generate 25,600 samples, and we compute the predicted class label of each image channel by a pre-trained MNIST classifier\footnote{https://github.com/fjxmlzn/PacGAN}. We report the number of modes for which at least one sample is generated. A higher mode number represents better results because it denotes that more modes are generated. We also report the Kullback-Leibler divergence (KL) between the generated samples and the real data over classes. A lower KL represents better results because it indicates that the generated distribution is more similar to the real one.

\begin{table}[h]
	\centering
	\renewcommand\arraystretch{1.15}
	\caption{Experiments on a static stacked MNIST dataset that contains 60,000 images. The results are averaged over 10 trials with standard error reported. Higher Modes and lower KL represent better results.}
	\begin{tabular}{c|c||r|r}
		\hline
		GDF&       LDF&          Modes& KL\\
		\hline
		&          & 876.8$\pm$3.69& 0.40$\pm$0.006\\
		\ding{52}&          & \textbf{984.5}$\pm$\textbf{1.35}& \textbf{0.20}$\pm$\textbf{0.006}\\
		& \ding{52}& 980.4$\pm$1.33& 0.28$\pm$0.008\\
		\hline
	\end{tabular}
	\label{tab1}
\end{table}

In TABLE~\ref{tab1}, the results illustrate two points. First, distribution fitting methods are effective in suppressing mode collapse. Learning a static dataset's distribution is easier, so our baseline also achieves good performance. However, with GDF and LDF, generated modes are still noticeably increased, and the KL is also lower. Second, the results with GDF are only slightly better than those with LDF, which indicates that LDF can match the real distribution. Therefore, LDF is also effective for combating mode collapse.

Second, we produce the stacked images in the training process to fairly compare with other methods. As mentioned, the real distribution is unknown in this case, so only LDF is utilized. 

\begin{table}[h]
	\centering
	\renewcommand\arraystretch{1.15}
	\caption{Experiments on stacked MNIST dataset. MD denotes Minibatch Discrimination. The last four results are averaged over 10 trials with standard error reported. Higher Modes and lower KL represent better results.}
	\begin{tabular}{l||r|r}
		\hline
		Method& Modes& KL\\
		\hline
		DCGAN~\cite{dcgan}&            99.0& 3.40\\
		ALI~\cite{ali}&            16.0& 5.40\\
		Unrolled GAN~\cite{unrolledgan}&            48.7& 4.32\\
		VEEGAN~\cite{veegan}&           150.0& 2.95\\
		MD~\cite{wgangp}&   24.5$\pm$7.67& 5.49$\pm$0.418\\
		PacDCGAN2~\cite{pacgan}& \textbf{1000.0}$\pm$\textbf{0.00}& \textbf{0.06}$\pm$\textbf{0.003}\\
		\hline
		Baseline&   24.1$\pm$4.10& 7.48$\pm$0.998\\
		Baseline w/ LDF&  970.0$\pm$2.28& 0.23$\pm$0.020\\
		
		\hline
	\end{tabular}
	
	\label{tab2}
\end{table}

In TABLE~\ref{tab2}, the first six rows are directly copied from~\cite{pacgan}. It is obvious that LDF achieves competitive performance. However, the performance does not reach the best, which is not as good as PacDCGAN2~\cite{pacgan}. We consider that it is because the batch size is not large enough. In LDF, the mean and standard deviation of generated data will be matched by those of a batch. Therefore, we consider that a large batch size will benefit LDF. To validate the hunch, we use different batch sizes to train the models. 

\begin{table}[h]
	\centering
	\renewcommand\arraystretch{1.15}
	\caption{Experiments on stacked MNIST dataset with different batch sizes. The results are averaged over 10 trials with standard error reported. Higher Modes and lower KL represent better results.}
	\begin{tabular}{c|l||r|r}
		\hline
		Batch&         \multirow{2}{*}{Method}& \multirow{2}{*}{Modes}& \multirow{2}{*}{KL}\\
		Size&                                 & & \\
		\hline
		\multirow{2}{*}{1x}&        Baseline&  24.1$\pm$4.10& 7.48$\pm$0.998\\
		& Baseline w/ LDF& 970.0$\pm$2.28& 0.23$\pm$0.020\\
		\hline
		\multirow{2}{*}{10x}&        Baseline& 157.0$\pm$71.50& 5.56$\pm$1.047\\
		& Baseline w/ LDF& 985.9$\pm$2.700& 0.20$\pm$0.026\\
		\hline
		\multirow{2}{*}{20x}&        Baseline& 160.6$\pm$18.27& 4.26$\pm$0.398\\
		& Baseline w/ LDF& 996.2$\pm$0.770& 0.13$\pm$0.011\\
		\hline
		\multirow{2}{*}{50x}&        Baseline& 184.0$\pm$18.27& 3.39$\pm$0.209\\
		& Baseline w/ LDF& \textbf{1000}$\pm$\textbf{0.000}& \textbf{0.03}$\pm$\textbf{0.002}\\
		\hline
	\end{tabular}
	\label{tab3}
\end{table}

\begin{table*}[t]
	\centering
	\renewcommand\arraystretch{1.1}
	\caption{Experiments on stacked MNIST dataset with different discriminators. The results are averaged over 10 trials with standard error reported. Higher Modes and lower KL represent better results.}
	\begin{tabular}{l||r|r|r|r}
		\hline
		\multirow{2}{*}{Method}& \multicolumn{2}{c|}{D is 1/4 size of G}& \multicolumn{2}{c}{D is 1/2 size of G}\\
		\cline{2-3}\cline{4-5}
		& Modes& KL& Modes& KL\\
		\hline
		Unrolled GAN~\cite{unrolledgan}, 0 step& 30.6$\pm$20.73& 5.99$\pm$0.42& 628.0$\pm$140.9& 2.58$\pm$0.75\\
		Unrolled GAN~\cite{unrolledgan}, 1 step& 65.4$\pm$34.75& 5.91$\pm$0.14& 523.6$\pm$55.77& 2.44$\pm$0.26\\
		Unrolled GAN~\cite{unrolledgan}, 5 steps& 236.4$\pm$63.30& 4.67$\pm$0.43& 732.0$\pm$44.98& 1.66$\pm$0.09\\
		Unrolled GAN~\cite{unrolledgan}, 10 steps& 327.2$\pm$74.67& 4.66$\pm$0.46& 817.4$\pm$37.91& 1.43$\pm$0.12\\
		\hline
		Baseline& 77.4$\pm$28.37& 8.66$\pm$1.34& 394.6$\pm$79.31& 2.40$\pm$0.64\\
		Baseline w/ LDF &\textbf{952.4}$\pm$\textbf{8.870}&\textbf{0.39}$\pm$\textbf{0.04}&
		\textbf{952.7}$\pm$\textbf{7.260}&\textbf{0.47}$\pm$\textbf{0.06}\\
		\hline
	\end{tabular}
	\label{tab5}
\end{table*}

In TABLE~\ref{tab3}, we see that with the increasing batch size, the performance will become better. It is proof of Corollary~\ref{cor1}. Nevertheless, the larger batch size is not able to completely solve the mode collapse problem as the mode number is still less than 200, even with a 50x batch size. We can catch all 1,000 modes with LDF. Thus, LDF shows effectiveness with a large batch size.

As we mentioned that the weight of the penalty term is set to 1, we also conduct experiments to prove that its value will not have critical impacts. For simplicity, the experiments are conducted on the static dataset with 60,000 images and GDF. As shown in TABLE~\ref{tab4}, we consider that there are only some fluctuations in training. Thus, we believe the value of $\lambda$ has no significant impact.

\begin{table}[h]
	\centering
	\renewcommand\arraystretch{1.15}
	\caption{Experiments on a static stacked MNIST dataset with different weights of the penalty term. The results are averaged over 10 trials with standard error reported. Higher Modes and lower KL represent better results.}
	\begin{tabular}{c||r|r}
		\hline
		$\lambda$ &          Modes& KL\\
		\hline
		~~0.5~~          & 981.7$\pm$1.81& 0.20$\pm$0.010\\
		~~1.0~~          & 984.5$\pm$1.35& 0.20$\pm$0.006\\
		~~2.0~~          & 983.2$\pm$1.04& 0.18$\pm$0.006\\
		\hline
	\end{tabular}
	
	\label{tab4}
\end{table}

To prove the stability of LDF, experiments similar with~\cite{unrolledgan} are conducted. We adjust the parameters of the discriminator to demonstrate the performance. In TABLE~\ref{tab5}, the first four rows are from~\cite{unrolledgan}. Obviously, our method achieves satisfactory performance, and it is still competitive with a weaker discriminator.

\begin{table*}[t]
	\centering
	\renewcommand\arraystretch{1.15}
	\caption{Experiments on CIFAR-10, STL-10 and Tiny-ImageNet. A higher IS and a lower FID represent better results. The symbol * represents the method designed for suppressing mode collapse. The symbol $\dagger$ represents the results for comparison obtained from a model trained by ourselves.}
	\begin{tabular}{l||c|c|c|c|c|c}
		\hline
		\multirow{2}{*}{Method}& \multicolumn{2}{c|}{CIFAR-10 ($32\times 32$)}& \multicolumn{2}{c|}{STL-10 ($32\times 32$)}& \multicolumn{2}{c}{Tiny-ImageNet ($64\times 64$)}\\
		\cline{2-7}
		& IS& FID& IS& FID& IS& FID\\
		\hline		   
		*MGGAN~\cite{mggan}&             6.608& 51.76&    -& -& -& -\\
		*MDGAN~\cite{mdgan}&                 -& 36.80&    -& -& -& -\\
		*D2GAN~\cite{d2gan}&   \textbf{7.150}$\pm$\textbf{0.070}&     -& 7.98& -& -& -\\
		*DRAGAN~\cite{dragan}&           6.900& 68.50&    -& -& -& -\\
		BEGAN~\cite{began}&              5.620& 71.40&    -& -& -& -\\
		WGAN~\cite{wgan}&      3.820$\pm$0.060&  55.90& $\dagger$3.764$\pm$0.042& $\dagger$77.13& $\dagger$3.419$\pm$0.044& $\dagger$78.79\\
		WGAN-GP~\cite{wgangp}& 6.680$\pm$0.060&  40.20& \textbf{8.420}$\pm$\textbf{0.130}& 55.10& $\dagger$5.557$\pm$0.109& $\dagger$66.45\\
		DCGAN~\cite{dcgan}&    $\dagger$6.474$\pm$0.045&  $\dagger$33.20& $\dagger$7.744$\pm$0.055& $\dagger$60.21& $\dagger$5.315$\pm$0.100& $\dagger$72.70\\
		GAN~\cite{gan}&        $\dagger$2.395$\pm$0.067&  $\dagger$89.98& $\dagger$2.400$\pm$0.044& $\dagger$93.31& $\dagger$2.537$\pm$0.099& $\dagger$90.01\\
		\hline
		*GAN w/ GDF (Ours)&    2.986$\pm$0.029&  85.78& 2.563$\pm$0.047& 82.24& 2.549$\pm$0.065& 88.44\\
		*GAN w/ LDF (Ours)&    2.904$\pm$0.030& 85.11& 2.498$\pm$0.015& 86.74& 2.537$\pm$0.087& 89.12\\
		*DCGAN w/ GDF (Ours)&  6.965$\pm$0.077& \textbf{30.03}& 7.991$\pm$0.095& \textbf{53.24}& \textbf{6.002}$\pm$\textbf{0.078}& \textbf{65.19}\\
		*DCGAN w/ LDF (Ours)&  6.798$\pm$0.081& 30.23& 7.826$\pm$0.079& 55.26& 5.875$\pm$0.089& 70.58\\
		\hline
	\end{tabular}
	\label{tab6}
\end{table*}

Finally, we give the convergence curves of the penalty term on the static dataset. From Eqs.~\ref{eq10} and~\ref{eq18}, GDF and LDF will not be involved in the adversarial training but only be minimized by the generator. Thus, the penalty terms are expected to be kept at small values. The curves in Fig.~\ref{fig_stacked} demonstrate that GDF and LDF are optimized as expected. Though LDF requires more iterations to estimate the parameters, it can finally be minimized similarly to GDF.


\begin{figure}[h]
	\centering
	\subfloat[]{
		\includegraphics[width=1.7in]{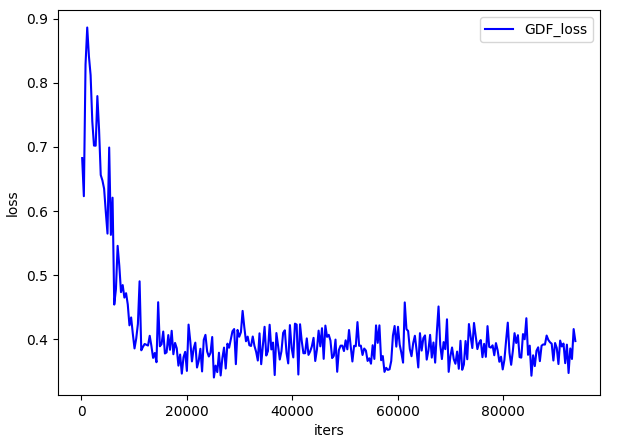}} 
	\subfloat[]{
		\includegraphics[width=1.7in]{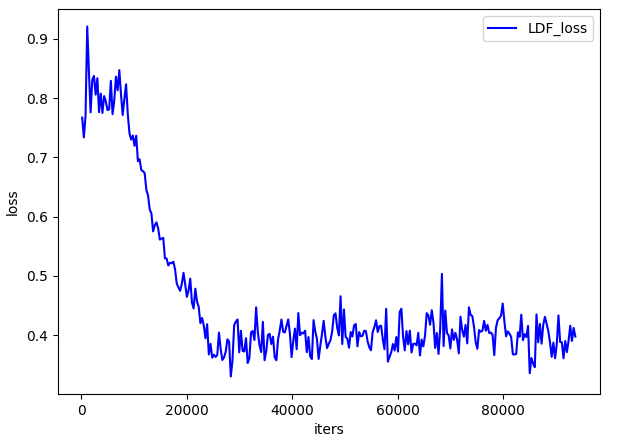}}
	\caption{The convergence curves of (a) GDF and (b) LDF. The horizontal coordinate axis represents the training iteration and the vertical one represents the values.}
	\label{fig_stacked}
\end{figure}

\subsection{CIFAR-10, STL-10 and Tiny-ImageNet Dataset}
Experiments on natural datasets are also necessary to demonstrate the effectiveness. We involve CIFAR-10~\cite{cifar10} and STL-10~\cite{stl10}, which are widely used, for our experiments. Additionally, considering our method is designed for combating mode collapse, we also involve Tiny-ImageNet~\cite{tinyimagenet}, which contains 200 classes, to show the performance of our method on the multi-class generation task. In the experiments, each network will generate 25,600 images for testing.

The Inception Score (IS)~\cite{is} and Fréchet Inception Distance (FID)~\cite{fid} are reported. IS can be calculated with
\begin{equation}
	IS(G)=exp(E_{x\sim p_g}D_{KL}(p(y|x)||p(y)))
\end{equation}
where $p_g$ is the distribution of generated data, $D_{KL}$ is the KL-divergence, $x$ is the input image and $y$ is the output of a pre-trained Inception V3 network~\cite{inceptionv3}. A higher IS represents better results. FID can be calculated with
\begin{equation}
	FID(y,g)=||\mu_{y}-\mu_{g}||^{2}_{2}+Tr(\Sigma_y+\Sigma_g-2(\Sigma_y\Sigma_g)^{\frac{1}{2}})
\end{equation}
where $y$ and $g$ are the outputs of a pre-trained Inception V3 network~\cite{inceptionv3} with the real and generated data as the inputs, respectively. $\mu$ and $\Sigma$ are the mean and covariance of the data distributions of $y$ and $g$. FID represents the distance between the two datasets. We will calculate the FID between the generated data and the testing data, and a lower FID represents better results. 

First, we evaluate our method with the original image resolutions. We also train some famous networks with a public implementation\footnote{https://github.com/eriklindernoren/PyTorch-GAN} for comparisons. The results are shown in TABLE~\ref{tab6}. We apply our GDF and LDF in GAN~\cite{gan} and DCGAN~\cite{dcgan} to demonstrate the effectiveness. We can observe that DCGAN with GDF achieves the best performance in most cases. Though the IS is not the best on CIFAR-10 and STL-10, the performance is still better than the original DCGAN. The comparisons prove that our method is effective. The convergence curves of GDF and LDF with DCGAN under different datasets can be found in Appendix.

Second, we try our method for higher image resolutions. As GAN~\cite{gan} and DCGAN~\cite{dcgan} are not designed for high-resolution generation, we involve StyleGAN2~\cite{stylegan2} to generate larger images. The results are shown in TABLE~\ref{tab6-1}.

\begin{table*}[t]
	\centering
	\renewcommand\arraystretch{1.15}
	\caption{Experiments on CIFAR-10, STL-10 and Tiny-ImageNet with higher resolutions. A higher IS and a lower FID represent better results.  }
	\begin{tabular}{l||c|c|c|c|c|c|c}
		\hline
		\multirow{2}{*}{Method}& \multirow{2}{*}{Size}& \multicolumn{2}{c|}{CIFAR-10}& \multicolumn{2}{c|}{STL-10}& \multicolumn{2}{c}{Tiny-ImageNet}\\ 
		\cline{3-8}
		& & IS& FID& IS& FID& IS& FID\\
		\hline
		StyleGAN2~\cite{stylegan2}& \multirow{3}{*}{$128\times 128$}& 
		8.266$\pm$0.094& 8.80& 8.660$\pm$0.212& 9.64& 10.082$\pm$0.222& 12.31\\
		StyleGAN2 w/ GDF (Ours)& & 8.517$\pm$0.092& 8.70& 8.801$\pm$0.217& 9.42& 10.134$\pm$0.219& 10.12\\
		StyleGAN2 w/ LDF (Ours)& & 8.398$\pm$0.101& 8.68& 8.777$\pm$0.202& 9.46& 10.134$\pm$0.237& 10.66\\
		\hline
		StyleGAN2~\cite{stylegan2}& \multirow{3}{*}{$256\times 256$}& 
		9.143$\pm$0.142& 9.68& 9.446$\pm$0.179& 10.46& 10.371$\pm$0.429&12.69\\
		StyleGAN2 w/ GDF (Ours)& & 9.221$\pm$0.118& 9.56& 9.508$\pm$0.152& 10.29& 10.522$\pm$0.375&10.04\\
		StyleGAN2 w/ LDF (Ours)& & 9.201$\pm$0.131& 9.61& 9.495$\pm$0.150& 10.46& 10.498$\pm$0.343&10.38\\
		\hline
	\end{tabular}
	\label{tab6-1}
\end{table*}

\subsection{Visualization}

The generated samples on MNIST~\cite{mnist}, CIFAR-10~\cite{cifar10}, STL-10~\cite{stl10} and Tiny-ImageNet~\cite{tinyimagenet} are visualized to give a intuitive result. 

\begin{figure}[h]
	\centering
	\includegraphics[width=3.3in]{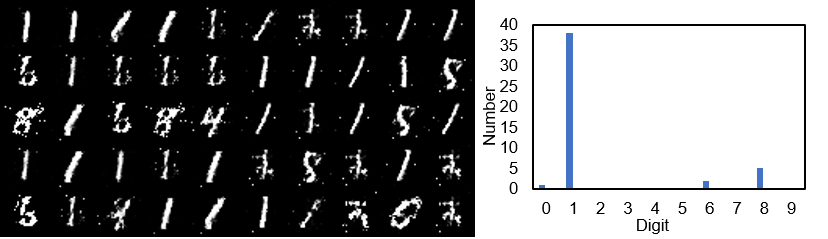}
	\caption{The samples in a mini-batch generated by GAN on MNIST, and the number of samples per class in the mini-batch. }
	\label{fig3}
\end{figure}

\begin{figure}[h]
	\centering
	\includegraphics[width=3.3in]{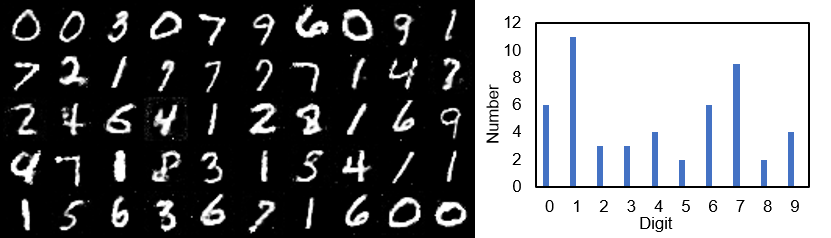}
	\caption{The samples in a mini-batch generated by GAN with GDF on MNIST, and the number of samples per class in the mini-batch.}
	\label{fig4}
\end{figure}

In Fig.~\ref{fig3}, we visualize the generated samples in a mini-batch of the original GAN~\cite{gan} framework. We see in the histogram that most samples contain the digit 1, which indicates that mode collapse occurs. In Fig.~\ref{fig4}, we visualize the generated samples of the framework trained with GDF, which illustrates that mode collapse has been suppressed.

\begin{figure*}[h]
	\centering
	\includegraphics[width=6.5in]{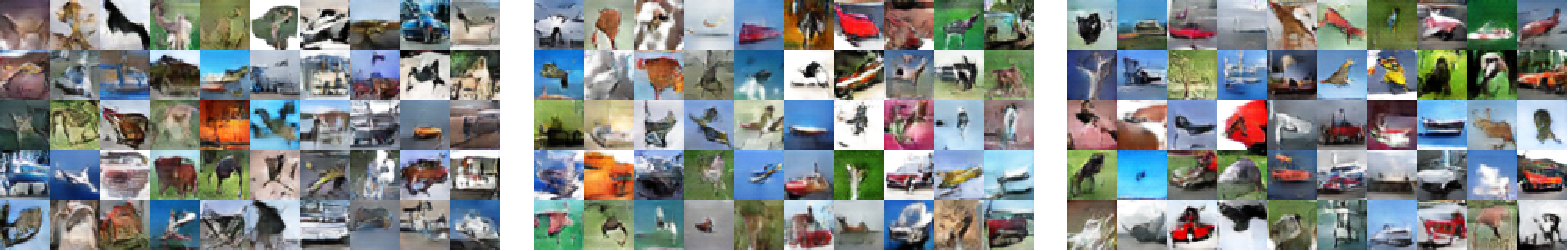}
	\caption{The samples in a mini-batch generated by DCGAN (left), DCGAN with LDF (middle), and DCGAN with GDF (right) on CIFAR-10. }
	\label{fig5}
\end{figure*}

\begin{figure*}[h]
	\centering
	\includegraphics[width=6.5in]{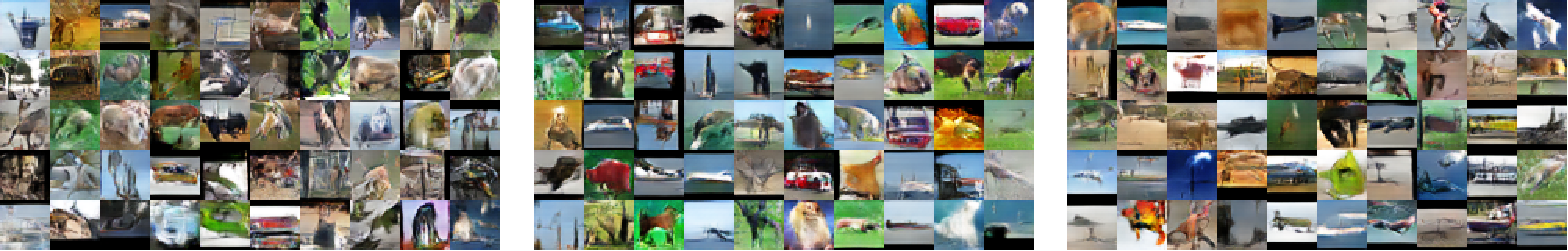}
	\caption{The samples in a mini-batch generated by DCGAN (left), DCGAN with LDF (middle), and DCGAN with GDF (right) on STL-10.}
	\label{fig6}
\end{figure*}

\begin{figure*}[h]
	\centering
	\includegraphics[width=6.5in]{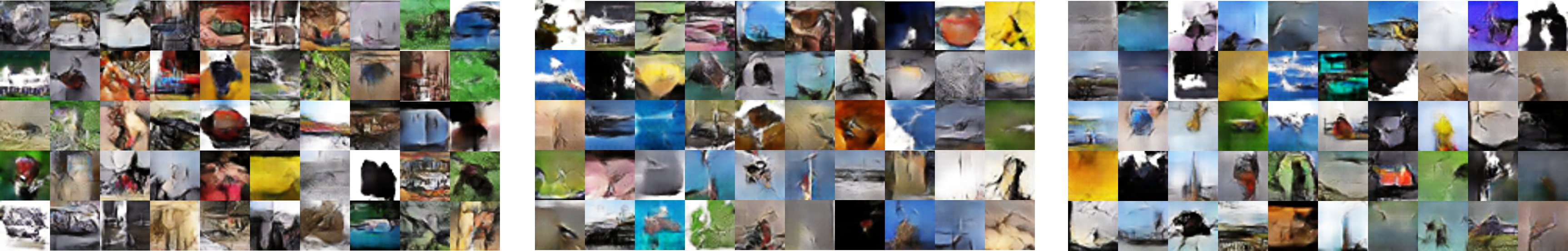}
	\caption{The samples in a mini-batch generated by DCGAN (left), DCGAN with LDF (middle), and DCGAN with GDF (right) on Tiny-ImageNet.}
	\label{fig7}
\end{figure*}

In Figs.~\ref{fig5},~\ref{fig6}, and~\ref{fig7}, some generated samples on CIFAR-10, STL-10, Tiny-ImageNet are visualized. It can be observed that DCGAN~\cite{dcgan} can reach a better performance with GDF.

\subsection{Efficiency}
We assert that the computation and time cost of GDF and LDF can almost be ignored. Thus, the floating point operations (FLOPs) and time cost are discussed. For GDF, the overall mean and standard deviation of real data are required, but they can be calculated before training. Then the values can be saved as a part of the dataset, so repeated calculation is not necessary. Therefore, the computation cost of GDF is less than LDF, and we will mainly discuss LDF.

In one training iteration, LDF requires the calculation of the means and standard deviations of a batch of generated data and all previous real data. Taking DCGAN~\cite{dcgan} trained on CIFAR-10~\cite{cifar10} as an example, we will try to compute the parameters. 

The generated data has four dimensions, and the size is (128, 3, 32, 32), representing the batch size, the number of channels, the width, and the height, respectively. First, the means will be calculated on the first dimension in our design, which contains 128 values for one mean. Therefore, it takes 127 add operations and 1 division operation. There are a total of $3\times 32\times 32=3072$ means, so the generated ones will take 393.216K FLOPs. Second, one standard deviation will also be calculated with 128 values. It takes 128 add operations to subtract the mean, 128 multiplication operations to get the square, 127 add operations to sum the results, 1 division operation, and 1 square root operation. There are also $3072$ standard deviations, so the generated ones will take about 1.183M FLOPs.

As for the real data, the previous means and standard deviations are recorded. With a new batch of data, one mean takes 1 multiplication operation to get the sum of previous values, 128 add operations to get the new sum, and 1 division operation. As a result, the means of the real data will take 399.36K FLOPs. One standard deviation will take 2 multiplication operations to get the previous sum, 128 add operations to subtract the mean, 128 multiplication operations to get the square, 128 add operations to get the new sum, 1 division operation, and 1 square root operation. As a result, the standard deviations of the real data will take about 1.192M FLOPs. 

In the training process, the means and standard deviations for real and generated data will both be calculated once in one iteration. To summarize, LDF will take about 3.167M FLOPs in one training iteration. As mentioned, GDF will take fewer FLOPs. In comparison, one front propagation of DCGAN~\cite{dcgan} will consume more than 14G FLOPs. Therefore, we argue that the computation cost of our method can almost be ignored. However, the efficiency may be affected by the batch size and number of data elements. The computation complexity is $O(mn)$ for $n$-tuple data where $m$ is the batch size. Thus, the complexity is proportional to the product of the batch size and the data element number.

To further demonstrate that the cost is very small, we record the training time of DCGAN~\cite{dcgan}, where CIFAR-10~\cite{cifar10} provides the data. The results are shown in TABLE~\ref{tab6-2}. It can be observed that the training time is not increased noticeably, which proves that the time cost of our method is small.

\begin{table}[h]
	\centering
	\renewcommand\arraystretch{1.15}
	\caption{Training time of DCGAN. The models are trained on CIFAR-10 for 200 epochs. Per Calc. represents the time for calculating the proposed penalty in one iteration.}
	\begin{tabular}{l||c|c|c|c}
		\hline
		Method& Total& Per Epoch& Per Iteration& Per Calc. \\ 
		\hline
		DCGAN~\cite{dcgan}& 5830.6s& 29.2s& 0.04s& 0.00ms\\
		DCGAN w/ GDF      & 5892.0s& 29.5s& 0.04s& 0.26ms\\
		DCGAN w/ LDF      & 5909.3s& 29.5s& 0.04s& 0.39ms\\
		\hline
	\end{tabular}
	\label{tab6-2}
\end{table}

\section{Discussion}
\label{sec:discussion}
We will discuss the advantages and limitations of our method. In the same way that for the commonly used regularization methods, the main advantage of our method is that no manual design is required. Many techniques for combating mode collapse tried to force the generator to generate diverse data. These approaches will need to be modified for a decent performance if a dataset contains data that is not diverse originally. Thus, the adaptability for various scenes of our method is better.

Compared to regularization methods, our method has two advantages. First, GDF and LDF only need a small amount of computation. The number of latent nodes in a network is usually much higher than the input layer. Therefore, compared to editing gradients in each latent node, our method requires less computation. Second, GDF and LDF are independent of the front propagation process, so they are easy to be transplanted. The distribution parameters are calculated independently, and only a penalty term is provided to be added to the training objective. In comparison, some regularization methods require changing the gradients in the front propagation process, which is more difficult to be applied in different frameworks.

Nevertheless, the limitation of our method is also obvious. First, because we utilize the mean and standard deviation of a batch to match those of the generated distribution, a large batch size must be applied. In some applications, the method will not work. For instance, for unpaired image-to-image translation, the batch size is usually set to 1 to achieve the best performance. For high-resolution image synthesis, sometimes, a single GPU can only deal with one image. Second, the mean and standard deviation only convey partial information, which is not enough to represent a distribution, especially for high-dimensional ones. Therefore, the applications may be limited, so we consider finding a better representation of data composition, e.g., using kernel density estimation, is valuable in the future. Last, we only demonstrate the effectiveness of our method when the sub-distributions are linearly combined. Although the experiments prove that it is most likely correct that most real distributions contain linear parts, the possible effectiveness of non-linearly combined mixed distributions is not thoroughly tested. In the future, exploring non-linearly combined mixed distributions may be valuable to refine our method.

\section{Conclusion}
\label{sec:conclusion}
In this work, we analyze the causes of mode collapse in GANs. Then we propose a distribution fitting method to suppress mode collapse. We present GDF and LDF to cope with different situations, and the experiments demonstrate the effectiveness of our method. The method is simple and effective, and only a small amount of computation is required with no manual design. However, the requirement of a large batch size limits the applications. In the future, we consider identifying alternative representations of data composition and utilizing it to achieve distribution fitting will be valuable. It is also possible to integrate the method into the virtual training criterion for a novel distance representation between distributions. To accurately estimate the real distribution parameters, one more research interest is exploring the minimum sample size for achieving GDF. The works will also help GANs to be applied in more research fields.

\section*{Acknowledgments}
This work was supported in part by the National Key Research and Development Program of China with ID 2018AAA0103203.


\appendix
{\appendices
	\section*{Network in Mixture of Gaussian Experiments}
	The network is the same as the one utilized in~\cite{unrolledgan}. We also provide the detailed architectures in TABLE~\ref{tab7} and~\ref{tab8}.
	\begin{table}[h]
		\centering
		\caption{Architecture of the generator in mixture Gaussian experiments.}
		\begin{tabular}{l||c|c}
			\hline
			Layers& Input Channel & Output Channel\\
			\hline
			Fully Connected & 256 & 128\\
			ReLU & - & -\\
			Fully Connected & 128 & 128\\
			ReLU & - & -\\
			Fully Connected & 128 & 2\\
			\hline
		\end{tabular}
		
		\label{tab7}
	\end{table}
	
	\begin{table}[h]
		\centering
		\caption{Architecture of the discriminator in mixture Gaussian experiments.}
		\begin{tabular}{l||c|c}
			\hline
			Layers& Input Channel & Output Channel\\
			\hline
			Fully Connected & 2 & 128\\
			ReLU & - & -\\
			Fully Connected & 128 & 128\\
			ReLU & - & -\\
			Fully Connected & 128 & 1\\
			Sigmoid & - & -\\
			\hline
		\end{tabular}
		\label{tab8}
	\end{table}
	\section*{Network in Stacked MNIST Experiments}
	We use the DCGAN-like network utilized in PacGAN~\cite{pacgan}, which is different from the original DCGAN~\cite{dcgan}. This network architecture can better represent mode collapse problems. The detailed architectures are shown in TABLE~\ref{tab9} and~\ref{tab10}. In the experiments of TABLE~\ref{tab5}, the output channel number of each layer of the discriminator will be multiplied by $\frac{1}{2}$ or $\frac{1}{4}$.
	
	\begin{table}[h]
		\centering
		\caption{Architecture of the generator in stacked MNIST experiments. k, s, and p represent kernel size, stride, and padding size. Channels show the input-output channels.}
		\begin{tabular}{l||c|l}
			\hline
			Layers& Configurations & Channels\\
			\hline
			Fully Connected & - & 256-1024\\
			Reshape & - & 1024-64\\
			Batch Normalization & - & -\\
			ReLU & - & -\\
			Transposed Convolution & \{k3,s2,p1\} & 64-32\\
			Batch Normalization & - & -\\
			ReLU & - & -\\
			Transposed Convolution & \{k3,s2,p1\} & 32-16\\
			Batch Normalization & - & -\\
			ReLU & - & -\\
			Transposed Convolution & \{k3,s2,p1\} & 16-8\\
			Batch Normalization & - & -\\
			ReLU & - & -\\
			Transposed Convolution & \{k3,s1,p1\} & 8-3\\
			Tanh & - & -\\
			\hline
		\end{tabular}
		\label{tab9}
	\end{table}
	
	\begin{table}[h]
		\centering
		\caption{Architecture of the discriminator in stacked MNIST experiments. k, s, and p represent kernel size, stride, and padding size. Channels show the input-output channels.}
		\begin{tabular}{l||c|l}
			\hline
			Layers& Configurations & Channels\\
			\hline
			Convolution & \{k3, s2, p1\} & 3-8\\
			Batch Normalization & - & -\\
			Leaky ReLU & 0.3 & -\\
			Convolution & \{k3, s2, p1\} & 8-16\\
			Batch Normalization & - & -\\
			Leaky ReLU & 0.3 & -\\
			Convolution & \{k3, s2, p1\} & 16-32\\
			Batch Normalization & - & -\\
			Leaky ReLU & 0.3 & -\\
			Reshape& - & 32-512\\
			Fully Connected & - & 512-1\\
			Sigmoid & - & - \\
			\hline
		\end{tabular}
		\label{tab10}
	\end{table}
	
	\section*{Training Setup of Experiments}
	In Section V.A, the training setup is as follows. The optimizers are Adam~\cite{adam}. For the generator, the learning rate is 0.001. For the discriminator, the learning rate is 0.0001. The batch size is 512, and the network is trained for 5,000 iterations. The experiments are carried out on a Linux platform with an Intel Core i7-8700 CPU, two USCORSAIR 16GB DDR4 RAMs, and an NVIDIA TITAN Xp GPU.
	
	In Section V.B, the training setup is as follows. The optimizers are Adam~\cite{adam} with a learning rate 0.0002. The original batch size is 128, and the network is trained for 200 epochs. The experiments are carried out on a Linux platform with an Intel Xeon(R) E5-2630 v3 CPU, four Samsung 16GB DDR4 RAMs, and three NVIDIA TITAN RTX GPUs.
	
	In Section V.C, the training setup is as follows.
	
	For GANs~\cite{gan}, DCGAN~\cite{dcgan}, WGAN~\cite{wgan} and WGAN-GP~\cite{wgangp}, the optimizers are Adam~\cite{adam}. The learning rate is 0.001 for GANs, DCGAN, and WGAN-GP and 0.00005 for WGAN. The batch size is 128, and the networks are trained for 200 epochs. The experiments are carried out on a Linux platform with an Intel Xeon(R) E5-2630 v3 CPU, four Samsung 16GB DDR4 RAMs, and three NVIDIA TITAN RTX GPUs.
	
	For StyleGAN2~\cite{stylegan2}, the optimizers are Adam~\cite{adam} with a learning rate 0.002. The batch size is 64, and the networks are trained for 200000 iterations. The experiments are carried out on a Linux platform with an Intel Xeon(R) E5-2630 v3 CPU, four Samsung 16GB DDR4 RAMs, and three NVIDIA TITAN RTX GPUs.
	
	\begin{figure}[b]
		\centering
		\subfloat[]{
			\includegraphics[width=1.7in]{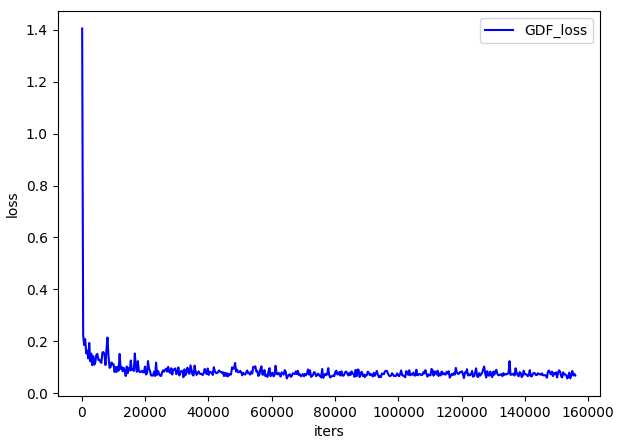}} 
		\subfloat[]{
			\includegraphics[width=1.7in]{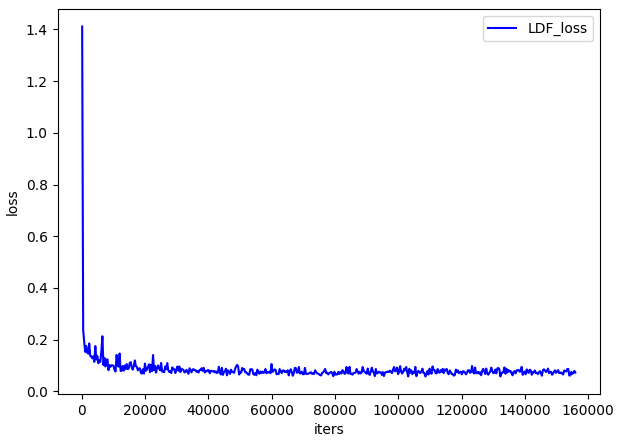}}
		\caption{The convergence curves of (a) GDF and (b) LDF on CIFAR-10. The horizontal coordinate axis represents the training iteration and the vertical one represents the values.}
		\label{fig_cifar}
	\end{figure}
	
	\section*{Convergence Curves on Natural Datasets}
	The convergence curves of GDF and LDF with DCGAN under different datasets are shown in Figs.~\ref{fig_cifar}, \ref{fig_stl10} and ~\ref{fig_tiny}. The penalty terms are minimized as expected, proving the effectiveness.
	
	\begin{figure}[t]
		\centering
		\subfloat[]{
			\includegraphics[width=1.7in]{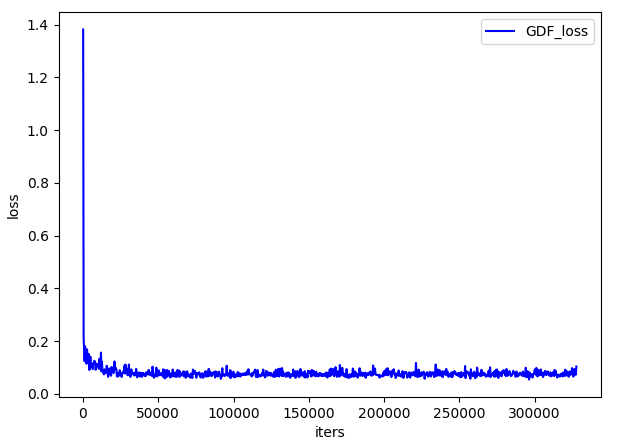}} 
		\subfloat[]{
			\includegraphics[width=1.7in]{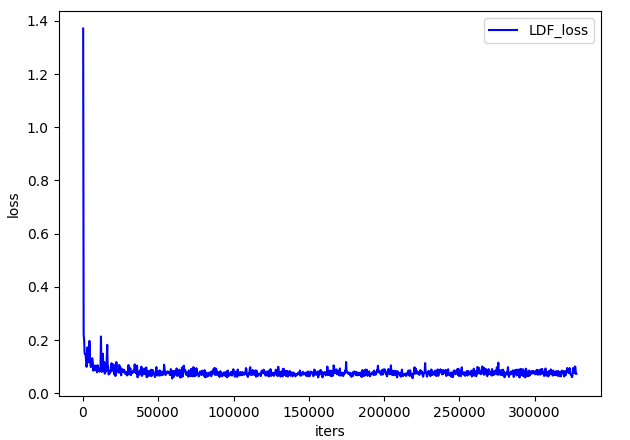}}
		\caption{The convergence curves of (a) GDF and (b) LDF on STL-10. The horizontal coordinate axis represents the training iteration and the vertical one represents the values.}
		\label{fig_stl10}
	\end{figure}
	
	\begin{figure}[t]
		\centering
		\subfloat[]{
			\includegraphics[width=1.7in]{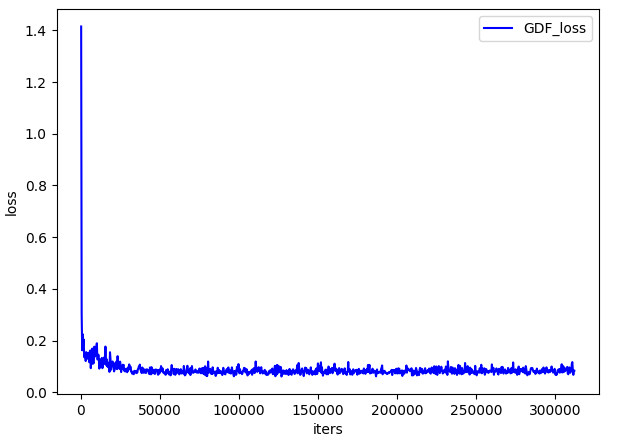}} 
		\subfloat[]{
			\includegraphics[width=1.7in]{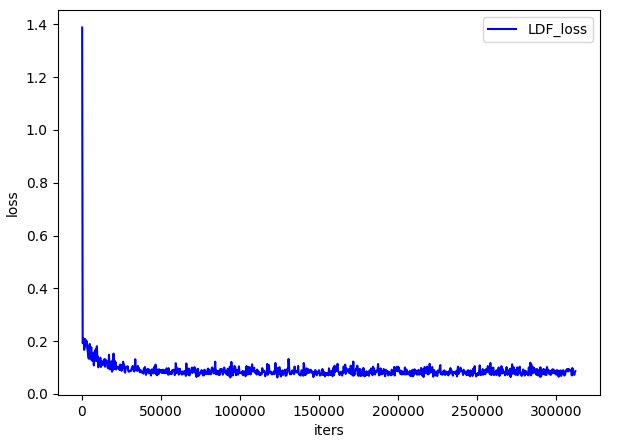}}
		\caption{The convergence curves of (a) GDF and (b) LDF on Tiny-ImageNet. The horizontal coordinate axis represents the training iteration and the vertical one represents the values.}
		\label{fig_tiny}
	\end{figure}
	
}


\bibliographystyle{plain}
\bibliography{IEEEabrv,refs}

\newpage

%
%
%
%

\vfill

\end{document}